\documentclass[letterpaper]{article}
\usepackage[utf8]{inputenc}
\usepackage[T1]{fontenc}
\usepackage{aaai17}
\usepackage{times}
\usepackage{helvet}
\usepackage{courier}
\usepackage{maltepaper}
\usepackage{amsmath}
\usepackage{amssymb}
\usepackage{amsthm}
\usepackage{stmaryrd}
\usepackage[mode=buildnew]{standalone}

\usepackage{todonotes}

\renewcommand{\models}{\vDash}
\newcommand{\notmodels}{\nvDash}

\settitle{Higher-Dimensional Potential Heuristics for Optimal Classical Planning}
\setauthor{%
    Florian Pommerening \and Malte Helmert\\
    University of Basel, Switzerland\\
    \{florian.pommerening, malte.helmert\}@unibas.ch%
    \And%
    Blai Bonet\\
    Universidad Sim{\'o}n Bol{\'i}var, Venezuela\\
    bonet@ldc.usb.ve%
}
\setkeywords{classical planning, cost-optimal planning, declarative heuristics, potential heuristics}
\setsubject{deterministic planning}
\setauthorname{Florian Pommerening, Malte Helmert, Blai Bonet}
\settrackingno{\# 1484}

\def\year{2017}

\newcommand{\sasplus}{\ensuremath{\textrm{SAS}^+}}
\newcommand{\variables}{\ensuremath{\mathcal V}}
\newcommand{\operators}{\ensuremath{\mathcal O}}
\newcommand{\init}{\ensuremath{s_{\textup{I}}}}
\newcommand{\goal}{\ensuremath{s_\star}}
\newcommand{\cost}{\ensuremath{\textit{cost}}}

\newcommand{\fact}[2]{\ensuremath{\langle #1, #2\rangle}}
\newcommand{\states}{\ensuremath{\mathcal S}}
\newcommand{\dom}{\ensuremath{\textit{dom}}}
\newcommand{\vars}{\ensuremath{\textit{vars}}}
\newcommand{\pre}{\ensuremath{\mathit{pre}}}
\newcommand{\eff}{\ensuremath{\mathit{eff}}}

\newcommand{\ts}{\ensuremath{\textup{TS}}}
\newcommand{\transitions}{\ensuremath{\mathcal T}}
\newcommand{\trans}[3]{\ensuremath{#1 \xrightarrow{#2} #3}}

\newcommand{\auxvar}{\ensuremath{\textit{Aux}}}

\newcommand{\res}[2]{\ensuremath{#1\llbracket#2\rrbracket}}

\newcommand{\astar}{\ensuremath{\textup{A}^{\ast}}}
\newcommand{\hopt}{\ensuremath{h^*}}
\newcommand{\hTCP}{\ensuremath{h^\textup{TCP}}}
\newcommand{\hOCP}{\ensuremath{h^\textup{OCP}}}
\newcommand{\hpot}{\ensuremath{h^\textup{pot}}}

\newcommand{\features}{\ensuremath{\mathcal F}}
\newcommand{\abstractions}{\ensuremath{\mathcal A}}

\newcommand{\inlinecite}[1]{\citeauthor{#1}~\shortcite{#1}}

\DeclareMathOperator*{\argmax}{arg\,max}

\newcommand{\ie}{i.\,e.}
\newcommand{\eg}{e.\,g.}

\newtheorem{definition}{Definition}
\newtheorem{theorem}{Theorem}
\newtheorem{proposition}{Proposition}
\newtheorem{corollary}{Corollary}

\renewenvironment{proof}{\noindent {\textbf{Proof:}}}%
                 {\hfill $\square$ \par \addvspace{\smallskipamount}}
               {\hfill $\square$ \par \addvspace{\smallskipamount}}

\begin{document}
\allowdisplaybreaks

\maketitle

\begin{abstract}
Potential heuristics for state-space search are defined as weighted
sums over simple state features. \emph{Atomic} features consider the
value of a single state variable in a factored state representation,
while \emph{binary} features consider joint assignments to two state
variables. Previous work showed that the set of all admissible and
consistent potential heuristics using \emph{atomic} features can be
characterized by a compact set of linear constraints. We generalize
this result to \emph{binary} features and prove a hardness result for
features of higher dimension. Furthermore, we prove a tractability
result based on the treewidth of a new graphical structure we call the
\emph{context-dependency graph}. Finally, we study the relationship of
potential heuristics to \emph{transition cost partitioning}.
Experimental results show that binary potential heuristics are
significantly more informative than the previously considered atomic
ones.
\end{abstract}

\section{Introduction}

Potential heuristics \cite{pommerening-et-al-aaai2015} are a family of
declarative heuristic functions for state-space search. They fix the
form of the heuristic to be a weighted sum over a set of features.
Conditions on heuristic values such as admissibility and consistency
can then be expressed as constraints on feature weights.

Previous work on admissible potential heuristics is limited to
\emph{atomic} features, which consider the value of a single state
variable in a factored state representation.
\inlinecite{pommerening-et-al-aaai2015} show that admissible and
consistent potential heuristics over atomic features can be
characterized by a compact set of linear constraints.
\inlinecite{seipp-et-al-icaps2015} introduce several objective
functions to select the \emph{best} potential heuristic according to
different criteria. They showed that a small collection of diverse
potential heuristics closely approximates the state equation heuristic
(SEQ) \cite{vandenbriel-et-al-cp2007,bonet-ijcai2013,%
  bonet-vandenbriel-icaps2014}. Since the quality of the SEQ heuristic
is an upper limit on the quality of \emph{any combination} of
potential heuristics over atomic features
\cite{pommerening-et-al-aaai2015}, more complex features are needed to
significantly improve heuristic quality of potential heuristics.

Additionally, theoretical analyses of common planning benchmarks
\cite{chen-gimenez-icaps2007,chen-gimenez-wollic2009,%
  lipovetzky-geffner-ecai2012,seipp-et-al-ijcai2016} suggest that
potential heuristics for \emph{binary} features (that consider the
joint assignment of two state variables rather than valuations over
single state variables as in the atomic case) could already lead to a
significant increase in accuracy in many planning domains.

In this paper we generalize known results about potential heuristics
with atomic features to those with larger features. After introducing
some notation, we show that admissible and consistent potential
heuristics for binary features are also characterized by a compact set
of linear constraints. We then prove that such a compact
characterization is not possible in the general case of features
mentioning three or more variables. However, we show that compact
representations are still possible for ``sparse'' features, as
measured by the treewidth \cite{dechter-2003} of a new graphical
structure we call the context-dependency graph. Finally, we generalize
a known relation between atomic potential heuristics and optimal cost
partitioning and show that potential heuristics correspond to optimal
transition cost partitionings \cite{keller-et-al-ijcai2016}, \ie, cost
partitionings that distribute the cost of each transition instead of
each operator.

\section{Background}

We consider \sasplus\ planning tasks \cite{backstrom-nebel-compint1995} in
transition normal form (TNF) \cite{pommerening-helmert-icaps2015}. A
planning task is a tuple $\Pi = \langle \variables, \operators, \init,
\goal \rangle$ with the following components.
$\variables$ is a finite set of \emph{variables}
where each $V \in \variables$ has a finite domain $\dom(V)$.
A pair $\fact{V}{v}$ of a variable $V\in \variables$ and one of its
values $v \in \dom(V)$ is called a \emph{fact}. Partial variable
assignments $p$ map a subset of variables $\vars(p) \subseteq \variables$
to values in their domain. Where convenient, we also treat them as
sets of facts. A partial variable assignment $s$ with $\vars(s) =
\variables$ is called a \emph{state} and $\states$ is the set of all
states. The state $\init$ is the \emph{initial state} of $\Pi$ and
the state $\goal$ is the \emph{goal state}. (Note that in TNF, there
is a single goal state.)
We call a partial variable assignment $p$ \emph{consistent} with a state
$s$ if $s$ and $p$ agree on all variables in $\vars(p)$. The set
$\operators$ is a finite set of \emph{operators} $o$, each
with a \emph{precondition} $\pre(o)$, an \emph{effect} $\eff(o)$, and a
\emph{cost} $\cost(o)$, where $\pre(o)$ and $\eff(o)$ are both partial
variable assignments and $\cost(o)$ is a non-negative integer.
The restriction to tasks in TNF means that we can assume that
$\vars(\pre(o)) = \vars(\eff(o))$. We denote this set of variables
by $\vars(o)$. Considering only tasks in TNF does not limit generality,
since there is an efficient transformation from general \sasplus\ tasks
into equivalent tasks in TNF \cite{pommerening-helmert-icaps2015}.

An operator $o$ is \emph{applicable} in state $s$ if $s$ is
consistent with $\pre(o)$. Applying $o$ in $s$ results in the state
$\res{s}{o}$ with $\res{s}{o}[V] = \eff(o)[V]$ for all $V \in \vars(o)$
and $\res{s}{o}[V] = s[V]$ for all other variables. An operator sequence
$\pi = \langle o_1, \dots, o_n \rangle$ is applicable in state $s$ if
there are states $s = s_0, \dots, s_n$ such that $o_i$ is applicable in
$s_{i-1}$ and $\res{s_{i-1}}{o_i} = s_i$. We write $\res{s}{\pi}$
for $s_n$. If $\res{s}{\pi} = \goal$, we call $\pi$ an \emph{$s$-plan}.
If $s = \init$, we call it a \emph{plan}. The cost of
$\pi$ under a cost function $\cost'$ is $\sum_{i=1}^n \cost'(o_i)$. An
$s$-plan $\pi$ with minimal $\cost'(\pi)$ among all $s$-plans is called
\emph{optimal} and we write its cost as $\hopt(s, \cost')$, or
$\hopt(s)$ if $\cost' = \cost$.

A \emph{heuristic function} $h$ maps states to values in $\mathbb R
\cup \{-\infty, \infty\}$. It is \emph{admissible} if $h(s) \le
\hopt(s)$ for all $s \in \states$, \emph{goal-aware} if $h(\goal) \le
0$, and \emph{consistent} if $h(s) \le \cost(o) + h(\res{s}{o})$ for
all $s \in \states$ and $o \in \operators$ that are applicable in $s$.

A task $\Pi$ induces a weighted, labeled transition system $\ts_\Pi =
\langle \states, \transitions, \init, \{\goal\}\rangle$ with the set of
states $\states$, the initial state $\init$, the single goal state
$\goal$ and set of transitions $\transitions$: for each
$s\!\in\! \states$ and $o\! \in\! \operators$ that is applicable in $s$,
there is a transition $\trans{s}{o}{\res{s}{o}}$
labeled with $o$ and weighted with $\cost(o)$. Shortest paths in
$\ts_\Pi$ correspond to optimal plans for $\Pi$.

A conjunction of facts is called a \emph{feature} and the number of
conjuncts is called its \emph{size}. Features of size 1 and 2 are
called \emph{unary} and \emph{binary} features. We say a
feature $f$ \emph{is true} in a state $s$ (written as $s \models f$) if
all its facts are in $s$. A \emph{weight function} for 
features $\features$ is a function $w: \features \to {\mathbb R}$. The
\emph{potential} of a state $s$ under a weight function $w$ is \[
  \varphi(s) = \sum_{f \in \features} w(f) [s \models f],
\] where the bracket is an indicator function \cite{knuth-amm1992}. We call $\varphi$ the \emph{potential heuristic}
for features $\features$ and weights $w$. Its \emph{dimension}
is the size of a largest feature $f\in\features$.

\section{Two-Dimensional Potential Heuristics}

Two-dimensional potential heuristics only consider atomic and binary
features. We use that consistent heuristics are admissible iff
they are goal-aware \cite{russell-norvig-1995}.

Let $\Pi = \langle \variables, \operators, \init, \goal, \cost \rangle$
be a planning task in TNF and $\ts_\Pi = \langle \states, \transitions,
\init, \{\goal\} \rangle$ its transition system. A potential heuristic
$\varphi$ over features $\features$ is goal-aware and consistent iff it
satisfies the following constraints:
\begin{align}
    \varphi(\goal) &\le 0, &
    \label{constraint:goal-awareness}\\
    \varphi(s) - \varphi(s') &\le \cost(o)
    & \text{for $\trans{s}{o}{s'} \in \transitions$.}
    \label{constraint:consistency}
\end{align}
This set of constraints has exponential size as there is one constraint
for each transition $\trans{s}{o}{s'}$ in $\transitions$.

Constraint~(\ref{constraint:goal-awareness}) is a linear constraint
over the weights, which is easy to see since
$\varphi(\goal)=\sum_{f\in\features} w(f)[\goal\models f]$.

Next, let $o\in\operators$ be a fixed operator and consider
constraint~(\ref{constraint:consistency}). Replacing $\varphi(s)$ and
$\varphi(s')$ by their definitions, we get the equivalent constraint
\begin{equation}
    \sum_{f \in \features} w(f)([s \models f] - [s' \models f]) \le \cost(o)
    \label{constraint:transition-consistency}
\end{equation}
for all transitions $\trans{s}{o}{s'} \in \transitions$.
We abbreviate the change of a feature's truth value ($[s \models f] - [\res{s}{o} \models f]$) as
$\Delta_o(f,s)$.

We partition the set of features into three subsets:
\emph{irrelevant} features $\features^\textup{irr}$ have no
variables in common with $\vars(o)$, \emph{context-independent} features
$\features^\textup{ind}$ mention only variables in $\vars(o)$,
and the remaining
\emph{context-dependent} features $\features^\textup{ctx}$ mention
one variable from $\vars(o)$ and another variable not in $\vars(o)$. We
write $\Delta_o^\textup{irr}(s)$ for $\sum_{f \in \features^\textup{irr}} w(f)
\Delta_o(f,s)$ and analogously $\Delta_o^\textup{ind}(s)$ and
$\Delta_o^\textup{ctx}(s)$.

The truth value of an irrelevant feature never changes by applying $o$
in some state. Thus, $\Delta_o^\textup{irr}(s) = 0$ for all states $s$.

For a context-independent feature $f$, the effect of applying $o$ in
$s$ is completely determined by $o$: $f$ holds in $s$ iff $f$ is entailed
by the precondition, and in $\res{s}{o}$ iff it is entailed
by the effect.
Thus, $\Delta_o(f,s)=[\pre(o)\models f]-[\eff(o)\models f]$
for every state $s$ in which $o$ is applicable.
Clearly, $\Delta_o(f,s)$ and $\Delta_o^\textup{ind}(s)$ do not depend on the state $s$
for $f \in\features^\textup{ind}$ and
we write $\Delta_o(f)$ and $\Delta_o^\textup{ind}$.

A feature $f \in \features^\textup{ctx}$ is a conjunction $f = f_o \land f_{\bar o}$
where $f_o$ is a fact over a variable in $\vars(o)$ and $f_{\bar o}$ is a fact
over a variable in $\variables_{\bar o} = \variables \setminus \vars(o)$. If $o$ is applied in a state $s$ with
$s \notmodels f_{\bar o}$, then $s\notmodels f$ and $\res{s}{o}\notmodels f$, so
$\Delta_o(f,s) = 0$. For the remaining features, we know that $f_{\bar o}$ is
present in both $s$ and $\res{s}{o}$ and the truth value of $f_o$ is solely determined by $o$.
Thus $\Delta_o(f, s) = \Delta_o(f_o) [s \models f_{\bar o}]$.
If $o$ is applicable in $s$,
\begin{align}
  \Delta_o^\textup{ctx}(s)
  &= \! \sum_{V\in\variables_{\bar o}}
  \sum_{\substack{f\in\features^\textup{ctx}\\f = f_o \land f_{\bar o}\\\vars(f_{\bar o}) = \{V\}}}
  w(f)\Delta_o(f,s) \\
  &= \! \sum_{V\in\variables_{\bar o}}
  \sum_{\substack{f\in\features^\textup{ctx}\\f = f_o \land f_{\bar o}}}
  w(f)\Delta_o(f_o)[f_{\bar o} \!=\! \fact{V}{s[V]}] \\
  &\leq \! \sum_{V\in\variables_{\bar o}}
  \sum_{\substack{f\in\features^\textup{ctx}\\f = f_o \land f_{\bar o}}}
  w(f)\Delta_o(f_o)[f_{\bar o} \!=\! \fact{V}{v_V^*]}]
\label{eq:important:1}
\intertext{where}
  \notag
  v_V^* &= \argmax_{v\in\dom(V)} \sum_{\substack{f\in\features^\textup{ctx}\\f = f_o \land f_{\bar o}}} w(f)\Delta_o(f_o)[f_{\bar o}=\fact{V}{v}] .
\end{align}
If we denote the inner sum in \eqref{eq:important:1} with $b^o_V$, then
\begin{align}
\varphi(s) - \varphi(s') \leq \Delta_o^{\textup{ind}} + \sum_{V\in\variables_{\bar o}} b^o_V
\label{eq:important:2}
\end{align}
for all transitions $\trans{s}{o}{s'}\in\transitions$. Therefore, if
$\Delta_o^{\textup{ind}} + \sum_{V\in\variables_{\bar o}} b^o_V \leq cost(o)$
for all operators $o$, then $\varphi$ is consistent.  Conversely, if
$\varphi$ is consistent, then $\varphi(s)-\varphi(\res{s}{o})\leq cost(o)$
for operator $o$ and the states $s$ in which $o$ is applicable.
In particular, for states $s^*$ such that $s^*[V]=\pre(o)[V]$
for $V\in\vars(o)$, and $s^*[V]=v_V^*$ otherwise. It is then not
difficult to check that the inequality in \eqref{eq:important:2}
is tight for such states $s^*$.
Hence, $\varphi$ is consistent iff
$\Delta_o^{\textup{ind}} + \sum_{V\in\variables_{\bar o}} b^o_V \leq cost(o)$
for all operators $o$.

Putting everything together, we see that the constraint~\eqref{constraint:transition-consistency}
for a fixed operator $o$ is equivalent to the constraints
\begin{align}
\label{constraint:transition-consistency-rewrite}
&\Delta_o^\textup{ind} + \sum_{V\in\variables_{\bar o}} z^o_V \le \cost(o), \\
&z^o_V \ge \!\!
  \sum_{\substack{f\in\features^\textup{ctx}\\ f = f_o\land\fact{V}{v}}} \!\!
  w(f)\Delta_o(f_o)
  \quad\text{for $V\!\in\!\variables_{\bar o}, v\!\in\!\dom(V)$}
\label{constraint:bound-ge-max}
\end{align}
where $z^o_V$ is a new variable that upper bounds $b^o_V$.
This set of constraints has $O(|\variables| d)$ constraints
where $d$ bounds the size of the variable domains, while
each constraint has size $O(|\features|+|\variables|)$.
The set of constraints is over the variables $\{w(f):f\in\features\}\cup\{z^o_V:o\in\operators,V\in\variables_{\bar o}\}$.

\begin{theorem}
\label{thm:tractable:binary}
Let $\features$ be a set of features of size at most 2
for a planning task $\Pi$. The set of solutions to the constraints
\eqref{constraint:goal-awareness}, and
\eqref{constraint:transition-consistency-rewrite}--\eqref{constraint:bound-ge-max}
for each operator, projected to $w$, corresponds to the set of weight
functions of admissible and consistent potential heuristics for $\Pi$
over $\features$.
The total number of constraints is $O(|\operators||\variables| d)$,
where $d$ bounds the size of variable domains, while each
constraint has size $O(|\features|+|\variables|)$.
\end{theorem}

\section{High-Dimensional Potential Heuristics}

In this section we show that a general result like
Theorem~\ref{thm:tractable:binary} is not possible,
unless NP equals P, for sets of features of dimension
3 or more, but we identify classes of problems on
which potential heuristics can be characterized
compactly.

\subsection{Intractability}

Theorem~\ref{thm:tractable:binary} allows one to answer
many interesting questions in polynomial time about
potential heuristics of dimension 2.
In particular, by solving a single LP one can test whether
a given potential heuristic is consistent and/or goal-aware.
We use this idea to show that no general result like
Theorem~\ref{thm:tractable:binary} is possible for potential
heuristics of dimension 3, by making a reduction of
non-3-colorability (a decision problem that is complete
for coNP \cite{garey-johnson-1979}) into the problem of testing whether a potential
heuristic of dimension 3 is consistent.

Let $G=\langle V,E\rangle$ be an undirected graph. We first construct,
in polynomial time, a planning task $\Pi=\langle \variables,
\operators, \init, \goal \rangle$ in TNF and a potential
heuristic $\varphi$ of dimension 3 such that $G$ is not 3-colorable iff
$\varphi$ is consistent.
The task $\Pi$ has $|V|+1$ variables: one variable $C_v$ for
the color of each vertex $v\in V$ that can be either red,
blue, or green, and one ``master'' binary variable denoted
by $M$.

For every vertex $v\in V$ and pair of different colors
$c,c'\in\dom(C_v)$, there is a unique operator $o_{v,c,c'}$
of zero cost that changes $C_v$ from $c$ to $c'$ when $M=0$.
For the variable $M$, there is a unique operator $o_M$, also
of zero cost, that changes $M$ from $0$ to $1$.
These are all the operators in the task $\Pi$.

Each state $s\in\states$ encodes a coloring of $G$, where
the color of vertex $v$ is the value $s[C_v]$ of the state
variable $C_v$.
The initial state $\init$ is set to an arbitrary coloring
but with the master variable set to 0; \eg, $\init[M]=0$
and $\init[C_v]=red$ for every vertex $v\in V$.
The goal state $\goal$ is also set to an arbitrary coloring
but with $\goal[M]=1$; \eg, $\goal[M]=1$ and $\goal[C_v]=red$
for every vertex $v\in V$.

The potential heuristic $\varphi$ of dimension 3 is constructed
as follows.  For features $f$ with $\vars(f)=\{M,C_u,C_v\}$ such
that $\{u,v\}\in E$ is an edge in the graph, let its weight
$w(f)=-1$ when $f[M]=1$ and $f[C_u]\neq f[C_v]$, and $w(f)=0$
otherwise.  For the feature $f_M=\fact{M}{1}$ of dimension 1,
let $w(f_M)=|E|-1$. The weight $w(f)$ for all other features $f$
is set to 0.

Let us now reason about the states of the task $\Pi$ and the
values assigned to them by the heuristic.
Let $s$ be a state for $\Pi$.
If $s[M]=0$, then $\varphi(s)=0$.
If $s[M]=1$, then no operator is applicable at $s$ and
$\varphi(s)\geq -1$, with $\varphi(s)=-1$ iff $s$ encodes a
3-coloring of $G$, as the feature $f_M$ contributes a value of
$|E|-1$ to $\varphi(s)$, while the features corresponding to
edges contribute a value of $-|E|$ when $s$ encodes a coloring.

Let us consider a transition $\trans{s}{o}{s'}\in\transitions$.
Clearly, $s[M]\leq s'[M]$ as no operator decreases the value of $M$.
If $s[M]=s'[M]=0$, then $\varphi(s)=\varphi(s')=0$.
If $s[M]=0$ and $s'[M]=1$, then $\varphi(s)\leq \varphi(s')$ iff $s'$
does not encode a 3-coloring of $G$.
The case $s[M]=s'[M]=1$ is not possible as no operator is applicable
in states with $s[M]=1$.
Therefore, since all operator costs are equal to zero, $\varphi$
is consistent iff there is no transition $\trans{s}{o}{s'}$ with
$s[M]=0$, $s'[M]=1$ and $s'$ encoding a 3-coloring of $G$, and the
latter iff the graph $G$ is not 3-colorable.

Finally, observe that testing whether a potential function $\varphi$ is
inconsistent can be done in non-deterministic polynomial time: guess a
state $s$ and an operator $o$, and check whether $\varphi(s) >
\varphi(\res{s}{o}) + \cost(o)$.

\begin{theorem}
\label{thm:intractability}
Let $\features$ be a set of features for a planning task $\Pi$,
and let $\varphi$ be a potential heuristic over $\features$.
Testing whether $\varphi$ is consistent is coNP-complete.
\end{theorem}

\subsection{Parametrized Tractability}

We first give an algorithm for maximizing a sum of functions using
linear programming, and then apply it to characterize high-dimensional
potential heuristics.

\subsubsection{Maximizing a Sum of Functions.}

Let $\mathcal X$ be a set of finite-domain variables. We extend the
notation $\dom(\mathcal X)$ to mean the set of variable assignments
over $\mathcal X$. For an assignment $\nu \in \dom(\mathcal X)$ and a
\emph{scope} $S \subseteq \mathcal X$, we use $\nu|_S$ to describe the
restriction of $\nu$ to $S$. Let $\Psi$ be a set of scoped functions
$\langle S, \psi \rangle$ with $S \subseteq \mathcal X$ and $\psi : S
\to \mathbb V$ for a set of values $\mathbb V$. For now, think of
$\mathbb V$ as the real numbers $\mathbb R$, but we will later
generalize this. We only require that maximization and addition is
defined on values in $\mathbb V$, that both operations are commutative
and associative, and that $\max\,\{a+c, b+c\} \equiv c + \max\,\{a, b\}$
for all $a,b,c \in \mathbb V$. We are interested in the value
$\textit{Max}(\Psi) = \max_{\nu \in \dom(\mathcal X)} \sum_{\langle S,
\psi \rangle \in \Psi} \psi(\nu|_S)$.

Computing $\textit{Max}(\Psi)$ is the goal of constraint optimization
for extensional constraints, an important problem in AI. It is
challenging because the number of valuations in $\dom(\mathcal X)$ is
exponential in the number $|\mathcal X|$ of variables. Bucket
elimination \cite{dechter-2003} is a well-known algorithm to compute
$\textit{Max}(\Psi)$. For reasons that will become clear later in this
section, we describe the bucket elimination algorithm in a slightly
unusual way: in our formulation, the algorithm generates a system of
equations, and its output can be extracted from the (uniquely defined)
solution to these equations. The system of equations makes use of
auxiliary variables $\auxvar_1, \dots, \auxvar_m$ that take values
from $\mathbb V$. The generated equations have the form $\auxvar_i =
\max_{j \in \{1, \dots, k_i\}} e_{i,j}$, where $e_{i,j}$ is a sum that
contains only values from $\mathbb V$ or the variables $\auxvar_1,
\dots, \auxvar_{i-1}$. Solutions to the system of equations guarantee
that $\auxvar_m \equiv \textit{Max}(\Psi)$.

\subsubsection{Bucket Elimination}

We now describe the general algorithm and state its correctness
(without proof due to lack of space). Its execution depends on an order
$\sigma = \langle X_1, \dots, X_n \rangle$ of the variables in
$\mathcal X$. The algorithm operates in stages which are enumerated
in a decreasing manner, starting at stage $n+1$ and ending at stage $0$:
\begin{itemize}
    \item Stage $n+1$ (Initialization).
        Start with a set $\{B_i\}_{i=0}^n$ of \emph{empty} buckets.
        Place each $\langle S, \psi \rangle \in \Psi$ into the bucket
        $B_i$ if $X_i$ is the largest variable in $S$, according to
        $\sigma$, or into the bucket $B_0$ if $S = \emptyset$.

    \item Stages $i=n,\ldots,1$ (Elimination).
        Let $\langle S_j, \psi_j \rangle$ for $j \in \{1, \dots, k_i\}$
        be the scoped functions currently in bucket $B_i$. Construct
        the scope $S_{X_i} = (\bigcup_{j \in \{1, \dots, k_i\}} S_j)
        \setminus \{X_i\}$ and the function $\psi_{X_i} : S_{X_i} \to
        \mathbb V$ that represents the contribution of all functions that
        depend on $X_i$. The definition of $\psi_{X_i}$ is added to the
        generated system of equations by adding one auxiliary variable
        $\auxvar_{X_i,\nu}$ for every $\nu \in \dom(S_{X_i})$ which
        represents the value $\psi_{X_i}(\nu)$:
        \[
            \auxvar_{X_i,\nu} =
            \max_{x_i \in \dom(X_i)}
            \sum_{j \in \{1, \dots, k_i\}}
            \psi_j(\nu_{x_i}|_{S_j})
        \]
        where $\nu_{x_i} = \nu \cup \{X_i \mapsto x_i\}$ extends the
        variable assignment $\nu$ with $X_i \mapsto x_i$. If $\psi_j$
        is a function in $\Psi$, then $\psi_j(\nu_{x_i}|_{S_j})$ is an
        element of $\mathbb V$. Otherwise $\psi_j$ is a previously
        defined function $\psi_{X_{i'}}$ for $i' > i$ and its value for
        $\nu' = \nu_{x_i}|_{S_j}$ is represented by $\auxvar_{X_{i'},\nu'}$.

        The newly defined function $\psi_{X_i}$ no longer depends on
        $X_i$ but depends on all variables in $S_{X_i}$, so $\langle
        S_{X_i}, \psi_{X_i} \rangle$ is added to bucket $B_j$ if $X_j$
        is the largest variable in $S_{X_i}$ according to $\sigma$ or
        to $B_0$ if $S_{X_i} = \emptyset$. Observe that $j < i$ because
        $S_{X_i}$ only contains variables from scopes where $X_i$ is
        the largest variable and $S_{X_i}$ does not contain $X_i$.

    \item Stage $0$ (Termination).
        Let $\langle S_j, \psi_j \rangle$ for $j \in \{1, \dots, k\}$
        be the scoped functions currently in bucket $B_0$. Add the
        auxiliary variable $\auxvar_\Psi$ and the equation
        $\auxvar_\Psi = \sum_{j \in \{1, \dots, k\}} \psi_j$
        analogously to the elimination step. (All $S_j$ are empty
        and the maximum is over $\dom(\emptyset) = \{\emptyset\}$.)
\end{itemize}

\noindent\emph{Example.} Consider $\Psi = \{\langle \{X\}, f \rangle,
\langle \{X, Y\}, g\rangle\}$ over the binary variables $\mathcal
X=\{X, Y\}$. Bucket elimination generates the following system of
equations for the variable order $\sigma = \langle X, Y\rangle$.
\begin{align*}
    \auxvar_1 = \auxvar_{Y,\{X \mapsto 0\}} &= \max_{y \in \{0,1\}} g(0,y) \\
    &= \max\, \{ g(0,0), g(0,1)\} \\
    \auxvar_2 = \auxvar_{Y,\{X \mapsto 1\}} &= \max_{y \in \{0,1\}} g(1,y)\\
    &= \max\, \{ g(1,0), g(1,1)\} \\
    \auxvar_3 = \auxvar_{X,\emptyset} &= \max_{x \in \{0,1\}} (f(x) +
    \auxvar_{Y,\{X \mapsto x\}})\\
    & = \max\,\{f(0) + \auxvar_1, f(1) + \auxvar_2\} \\
    \auxvar_4 = \auxvar_\Psi &= \auxvar_{X,\emptyset} = \auxvar_3
\end{align*}

\subsubsection{Bucket Elimination for Linear Expressions}

As a generalization of the bucket elimination algorithm, consider
$\mathbb V$ to be the following set $\mathbb E$ of mathematical
expressions over a set of variable symbols $\mathcal Y$. For every $Y
\in \mathcal Y$ and $r \in \mathbb R$, the expressions $Y$, $r$, and
$rY$ are in $\mathbb E$. If $a$ and $b$ are elements of $\mathbb E$,
then the expressions $(a + b)$ and $\max\,\{a, b\}$ are elements of
$\mathbb E$. There are no additional elements in $\mathbb E$. An
assignment $f : \mathcal Y \to \mathbb R$ that maps variables to values
can be extended to $\mathbb E$ in the straight-forward way. Two
expressions $a, b \in \mathbb E$ are equivalent if $f(a) = f(b)$ for
all assignments $f$. An expression is \emph{linear} if it does not
mention $\max$.

Clearly, maximization and addition are commutative, associative and satisfy
$\max\,\{a+c, b+c\} \equiv c + \max\,\{a, b\}$ for all expressions $a, b,
c \in \mathbb E$. Bucket elimination therefore generates a system of equations $\auxvar_i =
\max_{j \in \{1, \dots, k_i\}} e_{i,j}$, where all $e_{i,j}$ are sums
over expressions and variables $\auxvar_{i'}$ with $i' < i$. Since a variable
is a mathematical expression, the whole result can be seen as a system
of equations over the variables $\mathcal Y \cup \{\auxvar_1, \dots, \auxvar_m\}$.
If additionally all functions in $\Psi$ only produce linear expressions
over $\mathcal Y$, then in the resulting system all $e_{i,j}$ are
linear expressions over $\mathcal Y \cup \{\auxvar_1, \dots, \auxvar_m\}$.

Consider the example problem again. We define $f$ and $g$ so they map
to linear expressions over $\mathcal Y = \{a,b\}$:

\begin{center}
    \small
    \begin{tabular}{l|llcl|ll}
        $f(x)$ & $x = 0$ & $x = 1$ \\\hline
               & $3a - 2b$ & $4a + 2b$\\&&
    \end{tabular}
    \quad
    \begin{tabular}{l|ll}
        $g(x,y)$ & $x = 0$ & $x = 1$ \\\hline
        $y = 0$  & $8a$ & $-3b$ \\
        $y = 1$  & $7b$ & $0$
    \end{tabular}
\end{center}
In the resulting system of equations all elements in the maxima are
linear expressions over the variables $\{a, b, \auxvar_1, \auxvar_2,
\auxvar_3, \auxvar_4\}$:
\begin{align*}
    \auxvar_1 &= \max\,\{8a, 7b\}\\
    \auxvar_2 &= \max\,\{-3b, 0\}\\
    \auxvar_3 &= \max\,\{3a - 2b + \auxvar_1, 4a+2b + \auxvar_2\}\\
    \auxvar_4 &= \auxvar_3
\end{align*}
Bucket elimination guarantees that $\auxvar_4 \equiv \textit{Max}(\Psi)$ for
any value of $a$ and $b$ in $\mathbb{E}$.  We argue that this system of
equations can be solved by an LP solver.

\begin{theorem}
    Let $\mathcal Y$  and $\{\auxvar_1, \dots \auxvar_m\}$ be disjoint sets of
    variables. Let $P_{\max}$ be the system of equations $\auxvar_i = \max_{j
    \in \{1, \dots, k_i\}} e_{i,j}$ for $i \in \{1, \dots, m\}$, where
    $e_{i,j}$ is a linear expression over variables $\mathcal Y \cup
    \{\auxvar_1, \dots, \auxvar_{i-1}\}$. Let $P_{\textup{LP}}$ be the set of
    linear constraints $\auxvar_i \ge e_{i,j}$ for $i \in \{1, \dots, m\}$
    and $j \in \{1, \dots, k_i\}$.

    Every solution of $P_{\max}$ is a solution of $P_{\textup{LP}}$
    and for every solution $f$ of $P_{\textup{LP}}$ there is a solution
    $f'$ of $P_{\max}$ with $f'(Y) = f(Y)$ for all $Y \in \mathcal Y$
    and $f'(\auxvar) \le f(\auxvar)$ for all $\auxvar \notin \mathcal Y$.
\end{theorem}

As a corollary, we can use this result to represent a ``symbolic''
version of the bucket elimination algorithm with unknowns $\mathcal Y$
as an LP. (Note that the constraints generated by the bucket
elimination algorithm have exactly the form required by the theorem if
functions in $\Psi$ produce linear expressions.) This LP has the
property that for every assignment to the unknowns $\mathcal Y$ there
exists a feasible solution, and the values of $\auxvar_m$ in these
feasible solutions are exactly the set of numbers greater or equal to
$\textit{Max}(\Psi)$ for the given assignment to $\mathcal Y$. (For
simplicity, it would be preferable if \emph{only} $\textit{Max}(\Psi)$
itself resulted in a feasible assignment to $\auxvar_m$, but we will
see that the weaker property where $\auxvar_m$ may overestimate
$\textit{Max}(\Psi)$ is sufficient for our purposes.)

We denote the set of constraints for this LP by $P_{\textup{LP}(\Psi, \sigma)}$.
This LP can be solved in time that is polynomial in the size of
$P_{\textup{LP}(\Psi, \sigma)}$, so to bound the complexity, we
have to consider the number and size of the constraints in $P_{\textup{LP}(\Psi, \sigma)}$.

\inlinecite{dechter-2003} defines the \emph{dependency graph} of a
problem $\Psi$ over variables $\mathcal X$ as the undirected graph
$G(\Psi) = \langle \mathcal X, E \rangle$ with set of vertices given
by the variables $\mathcal X$ and an edge $\langle X,X' \rangle \in E$
iff $X \neq X'$ and there is a scoped function $\langle S, \psi
\rangle$ in $\Psi$ with $\{X, X'\} \subseteq S$. Given an undirected
graph $G$ and an order of its nodes $\sigma$, a \emph{parent} of a
node $n$ is a neighbor of $n$ that precedes $n$ in $\sigma$.
\citeauthor{dechter-2003} defines the \emph{induced graph} of $G$
along $\sigma$ as the result of processing each node of $G$ in
descending order of $\sigma$ and for each node connecting each pair of
its parents if they are not already connected. The induced width of
$G$ along $\sigma$ then is the maximal number of parents of a node in
the induced graph of $G$ along $\sigma$.

If there are $n$ variables in $\mathcal X$ and each of their domains
is bounded by $d$, then eliminating variable $X_i$ adds one equation
$\auxvar_{X_i,\nu} = \max_{j \in \dom(X_i)} e_{i,j}$ for each
valuation $\nu$ of the scope $S_{X_i}$ (variables relevant for $X_i$).
The size of this scope is limited by the induced width $w(\sigma)$, so
the number of valuations is limited by $d^{w(\sigma)}$. As there are
$n$ buckets to eliminate, the number of auxiliary variables in the LP
can thus be bounded by $O(nd^{w(\sigma)})$. Each such variable occurs
in $|\dom(X_i)| \le d$ constraints of the form $\auxvar_{X_i,\nu} \ge
e_{i,j}$ in $P_{\textup{LP}(\Psi, \sigma)}$, so there are
$O(nd^{w(\sigma) + 1})$ constraints.

\begin{theorem}
    \label{thm:ve}
    Let $\Psi$ be a set of scoped functions over the variables in
    $\mathcal X$ that map to linear expressions over $\mathcal Y$.
    Let $\sigma$ be an ordering for $\mathcal X$.

    Then $P_{\textup{LP}(\Psi, \sigma)}$ has $O(|\mathcal Y| + |\mathcal
    X|d^{w(\sigma)})$ variables and $O(|\mathcal X|d^{w(\sigma) + 1})$
    constraints, where $d = \max_{X\in\mathcal X} |\dom(X)|$ and
    $w(\sigma)$ is the induced width of $G(\Psi)$.
\end{theorem}

The smallest possible induced width of $G(\Psi)$ along any order
$\sigma$ is called the induced width of $G(\Psi)$ and equals the
treewidth of $G(\Psi)$ \cite{dechter-2003}. Unfortunately, finding the
induced width or a minimizing order is NP-hard. However, it is
fixed-parameter tractable \cite{downey-fellows-1999} with the
treewidth as the parameter \cite{bodlaender-sicomp1996}.

\subsubsection{High-Dimensional Potential Functions.}

Of the two conditions that characterize goal-aware and consistent
potential heuristics, constraint~\eqref{constraint:consistency} is
the most challenging to test as it really denotes an exponential
number of constraints, one per state. The constraint is
equivalent to
\begin{align*}
    cost(o)
    &\ge \max_{s\models\pre(o)} (\varphi(s) - \varphi(\res{s}{o}))\\
    &= \max_{s\models\pre(o)} \sum_{f\in\features} w(f)\Delta_o(f,s)\\
    &= \max_{\nu \in \dom(\variables_{\bar o})} \sum_{f\in\features} w(f)\Delta_o(f,s_{o,\nu})
\end{align*}
for all operators $o\in\operators$ where the partial state $s_{o,\nu}$ is $s_{o,\nu} = \pre(o)\cup \nu$.

As done before, for a fixed operator $o\in\operators$, we partition
$\features$ as $\features = \features^\textup{irr} \cup
\features^\textup{ind} \cup \features^\textup{ctx}$ and consider the
functions $\Delta^\textup{irr}_o(s)$, $\Delta^\textup{ind}_o(s)$ and
$\Delta^\textup{ctx}_o(s)$ for states $s$ on which the operator $o$ is
applicable. We have $\Delta^\textup{irr}_o(s) = 0$ and
$\Delta^\textup{ind}_o(s) = \Delta^\textup{ind}_o$ as the value of the
latter is independent of the state $s$. Therefore, the
constraint~\eqref{constraint:consistency} is equivalent to
\begin{align}
    \label{constraint:consistency:2}
    \cost(o)
    &\ge \Delta^\textup{ind}_o +
    \max_{\nu \in \dom(\variables_{\bar o})} \Delta^\textup{ctx}_o(s_{o,\nu})
\end{align}
where $\Delta^\textup{ctx}_o(s_{o,\nu})=\sum_{f\in\features^\textup{ctx}}
w(f)\Delta_o(f,s_{o,\nu})$. Observe that for $f\in\features^\textup{ctx}$,
the value of $\Delta_o(f,s_{o,\nu})$ is equal to
\begin{align*}
    [\pre(o)\cup s_{o,\nu}|_{\variables_f\setminus\variables_o} \models f] -
    [\eff(o)\cup\res{s_{o,\nu}}{o}|_{\variables_f\setminus\variables_o}
    \models f] .
\end{align*}
Let us define the functions $\psi^f_o$ that maps partial assignments
$\nu \in \dom(\variables_f \setminus \variables_o)$ to expressions in
$\{0, w(f), -w(f)\}$ as
\[
    \psi^f_o(\nu) = w(f)([\pre(o) \cup \nu \models f] -
                         [\eff(o) \cup \nu \models f]).
\]
For operator $o\in\operators$, we define $\Psi_o=\{\psi^f_o : f \in
\features^\textup{ctx}\}$. Then
\begin{align*}
    \max_{\nu \in \dom(\variables_{\bar o})} \Delta^\textup{ctx}_o(s_{o,\nu})
    &=
    \max_{\nu \in \dom(\variables_{\bar o})}
    \sum_{f \in \features^\textup{ctx}}
    \psi^f_o(s_{o,\nu}|_{\variables_f \setminus \variables_o})\\
    &=
    \textit{Max}(\Psi_o).
\end{align*}
Constraint \eqref{constraint:consistency:2} is then equivalent to
\begin{align*}
    \cost(o) \ge
    \Delta^\textup{ind}_o + \textit{Max}(\Psi_o).
\end{align*}

Applying Theorem~\ref{thm:ve} to $\Psi_o$ along an ordering $\sigma_o$
for the variables $\variables$ in $\Pi$, we obtain a set of constraints
$P_{\textup{LP}(\Psi_o, \sigma_o)}$ that characterize
$\textit{Max}(\Psi_o)$. We can replace the above constraint by
$\cost(o) \ge \Delta^\textup{ind}_o + \auxvar_{\Psi_o}$ and the constraints
in $P_{\textup{LP}(\Psi_o, \sigma_o)}$. The constraints
$P_{\textup{LP}(\Psi_o, \sigma_o)}$ allow setting $\auxvar_{\Psi_o}$ higher
than necessary, but this is never beneficial as long as $\auxvar_{\Psi_o}$
does not occur in the objective.

\begin{theorem}
    \label{thm:tractable:general}
    Let $\features$ be a set of features for a planning task $\Pi$, and
    let $\{\sigma_o\}_{o\in\operators}$ be an indexed collection of
    orderings of the variables in $\Pi$ (one ordering $\sigma_o$ for
    each operator $o\in\operators$). Then, the set of solutions to
    \begin{align}
        \varphi(\goal) &\le 0, &
        \label{constraint:gen:goal-awareness} \\
        \Delta^\textup{ind}_o + \auxvar_{\Psi_o} &\le \cost(o), &
        \text{for each operator $o\in\operators$,}
        \label{constraint:gen:consistency:1} \\
        &P_{\textup{LP}(\Psi_o, \sigma_o)} &
        \text{for each operator $o\in\operators$}
        \label{constraint:gen:consistency:2}
    \end{align}
    (projected on the feature weights) corresponds to the set of weight
    functions of admissible and consistent potential heuristics for
    $\Pi$ over $\features$.
\end{theorem}

We finish the section by bounding the number and size of the
constraints in $P_{\textup{LP}(\Psi_o,\sigma_o)}$. We define the
\emph{context-dependency graph} $G(\Pi, \features, o)$ for a planning
task $\Pi$, features $\features$ and an operator $o$ as follows: the
vertices are the variables of $\Pi$ and there is an edge between $V$
and $V'$ with $V \neq V'$ iff there is a feature $f\in\features$ with
$\vars(f) \cap \vars(o) \neq \emptyset$ and
$\{V,V'\}\subseteq\vars(f)\setminus\vars(o)$.

\begin{theorem}
    Let $\features$ be a set of features for a planning task $\Pi$, let
    $o$ be an operator of $\Pi$, and let $\sigma_o$ an ordering on the
    variables in $\Pi$. The number of constraints in
    $P_{\textup{LP}(\Psi_o,\sigma_o)}$ is $O(nd^{w(\sigma_o)+1})$ where
    $n$ is the number of variables, $d$ bounds the size of the variable
    domains, and $w(\sigma_o)$ is the induced width of
    $G(\Pi,\features,o)$ along the ordering $\sigma_o$. The number of
    variables in $P_{\textup{LP}(\Psi_o,\sigma_o)}$ is $O(|\features| +
    nd^{w(\sigma_o)})$.
\end{theorem}

By combining the constraints from the different operators according to
Theorem~\ref{thm:tractable:general}, we thus obtain the following
fixed-parameter tractability result.
\begin{corollary}
  Let $\features$ be a set of features for a planning task $\Pi$.
  Define the parameter $w^*$ as the maximum treewidth of all
  context-dependency graphs for $\Pi$ and $\features$.

  Computing a set of linear constraints that characterize the
  admissible and consistent potential heuristics with features
  $\features$ for $\Pi$ is fixed-parameter tractable with parameter
  $w^*$.
\end{corollary}

We remark that the general result (again) implies a polynomial
characterization of potential heuristics where all features have
dimension at most 2. In this case, no feature can simultaneously
include a variable mentioned in a given operator $o$ and two further
variables not mentioned in $o$, and hence all context-dependency
graphs are devoid of edges. In edge-free graphs, all orderings have
width $0$, and hence for each operator $o$,
$P_{\textup{LP}(\Psi_o,\sigma_o)}$ has $O(nd)$ constraints and
$O(|\features| + n)$ variables. The parameter $w^*$ is 0 in this case.

\begin{figure*}
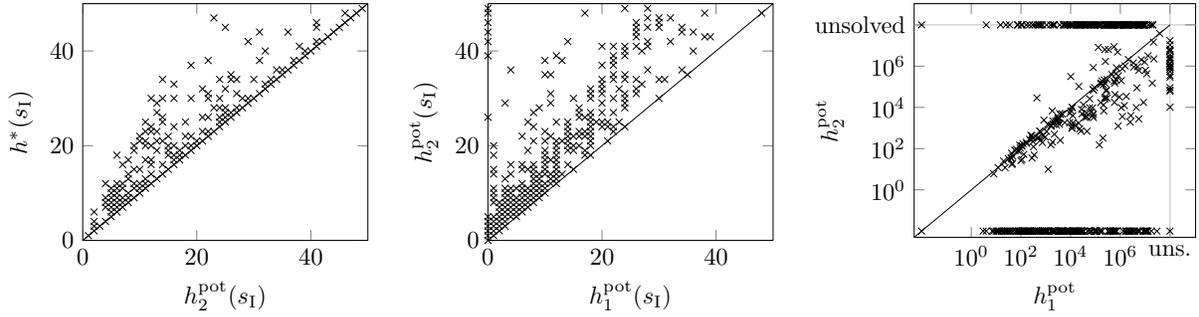

    \centering
    \resizebox{!}{1.65in}{\includestandalone{scatter-initial-h-hopt}}
    \quad
    \resizebox{!}{1.65in}{\includestandalone{scatter-initial-h}}
    \quad
    \resizebox{!}{1.65in}{\includestandalone{scatter-expansions}}
    \caption{
        Initial heuristic values below 50 of $\hpot_1$, $\hpot_2$, and
        $\hopt$. The final plot shows the number of expansions in an
        $\astar$ search with $\hpot_1$ and $\hpot_2$ (without
        expansions in the last $f$-layer).
    }
    \label{fig:initial-h}
\end{figure*}

\section{Relation to Cost Partitioning}

Operator cost partitioning \cite{katz-domshlak-icaps2007wshdip,%
yang-et-al-jair2008,katz-domshlak-aij2010} is a technique to make the
sum of several heuristics admissible by distributing the cost of each
operator between them. \inlinecite{katz-domshlak-aij2010} show how the
\emph{optimal operator cost partitioning} (OCP) can be computed in
polynomial time for a large family of abstraction heuristics.
\inlinecite{pommerening-et-al-aaai2015} show that an admissible and
consistent potential heuristic over all atomic features that achieves
the maximal initial heuristic value corresponds to an OCP over atomic
projection heuristics. Here, we extend this result to all potential
heuristics that use the abstract states of a set of abstractions as
features. To that end, we first discuss a generalization of OCP, which
we call \emph{optimal transition cost partitioning} (TCP)
\cite{keller-et-al-ijcai2016}\footnote{\citeauthor{keller-et-al-ijcai2016}
use \emph{state-dependent cost partitioning}, which may be confused
with partitionings that are re-optimized for each state.}.

\begin{definition}
    Let $\Pi$ be a planning task with cost function $\cost$ and
    transitions $\transitions$. A set of \emph{transition cost
    functions} $\cost_i : \transitions \to {\mathbb R}$,  $1 \le i
    \le n$, is a \emph{transition cost partitioning} if 
    \[
        \sum_{i=1}^n \cost_i(\trans{s}{o}{s'}) \le \cost(o) \qquad
        \text{for all $\trans{s}{o}{s'} \in \transitions$.}
    \]
\end{definition}

To use cost partitionings, the notion of operator cost
functions must be extended to cost functions with possibly
negative values. For example, the cost of an $s$-plan $\pi = \langle
o_1, \dots, o_n \rangle$ under transition cost function $\cost'$ is
$\sum_{i=1}^n \cost'(\trans{s_{i-1}}{o_i}{s_i})$, where $s=s_0, \dots,
s_n = \res{s}{\pi}$ are the states visited by the plan $\pi$.
The cheapest plan cost under $\cost'$ can now be negative or even
$-\infty$ (in the case of negative cost cycles).

\begin{proposition}[\inlinecite{keller-et-al-ijcai2016}]
    Let $\Pi$ be a planning task, $P = \langle\cost_1, \dots,
    \cost_n\rangle$ be a transition cost partitioning for $\Pi$, and
    $h_1, \dots, h_n$ be admissible heuristics for $\Pi$. The
    heuristic $h_P(s) = \sum_{i=1}^n h_i(s, \cost_i)$ is admissible.
\end{proposition}


\begin{definition}
    Let $\Pi$ be a planning task, $h_1, \dots, h_n$ be admissible
    heuristics for $\Pi$, and $\mathcal P$ be the set of all transition
    cost partitionings. The \emph{optimal
    transition cost partitioning heuristic} $\hTCP$ is
    $\hTCP(s) = \max_{P\in \mathcal P} h_P(s)$.
\end{definition}

Operator cost partitionings and the OCP heuristic are special cases
where all cost functions satisfy $\cost(t_1) = \cost(t_2)$ for all
pairs of transitions $t_1, t_2$ labeled with the same operator. The TCP
heuristic thus dominates the OCP heuristic. The paper by
\citeauthor{keller-et-al-ijcai2016} contains examples where this
dominance is strict.

The linear program defined by \inlinecite{katz-domshlak-aij2010} to
compute an optimal OCP for a set of abstraction
heuristics can be extended to transition cost partitionings.

Formally, an abstraction heuristic is based on a homomorphic mapping
$\alpha: \states \to \states^\alpha$ that maps states of a planning
task $\Pi$ to \emph{abstract states} of an \emph{abstract transition
system} $\ts^\alpha = \langle \states^\alpha, \transitions^\alpha,
\alpha(\init), \{\alpha(\goal)\} \rangle$, such that for every
transition $\trans{s}{o}{s'}$ of $\Pi$, $\transitions^\alpha$ contains
the transition $\trans{\alpha(s)}{o}{\alpha(s')}$ with the same weight.
%
For a collection $\abstractions$ of abstractions, the optimal TCP can
be encoded using two kinds of variables. The variable $h(s)$ represents
the goal distance for each abstract state $s \in \states^\alpha$ of
each $\alpha \in \abstractions$ (we assume the states are uniquely
named). The variable $c^\alpha(t)$ represents the cost of a transition
$t \in \transitions$ that is attributed to abstraction $\alpha$.
Using linear constraints over these variables, we can express that the
variables $h(s)$ do not exceed the true goal distance under the cost
function encoded in $c^\alpha$ and that the cost functions respect the
cost partitioning property:
\begin{align}
    h(\alpha(\goal)) &= 0
    && \text{for $\alpha \in \abstractions$}
    \label{constraint:tcp-goal}\\
    h(\alpha(s)) - h(\alpha(s')) &\le c^\alpha(\trans{s}{o}{s'})
    && \parbox[c][4ex][c]{22mm}{\text{for $\alpha \in \abstractions$}\\
    \text{and $\trans{s}{o}{s'} \in \transitions$}}
    \label{constraint:tcp-consistency}\\
    \sum_{\alpha \in \abstractions} c^\alpha(\trans{s}{o}{s'}) &\le \cost(o)
    && \text{for $\trans{s}{o}{s'} \in \transitions$}
    \label{constraint:tcp-cost-partitioning}
\end{align}

For a set of abstractions $\abstractions$, we define the set
$\features_\abstractions = \bigcup_{\alpha \in \abstractions}
\states^\alpha$ of features for $\abstractions$ with the interpretation that a state $s$ has the
feature $s' \in \states^\alpha$ iff $\alpha(s) = s'$. In the special
case where $\alpha$ is a projection to $k$ variables, the features
$\states^\alpha$ correspond to conjunctions of $k$ facts. We want to
show that TCP heuristics over $\abstractions$ and admissible and
consistent potential heuristics over features $\features_\abstractions$
have the same maximal heuristic value.

\begin{proposition}
    Let $s$ be a state of a planning task $\Pi$ and $\abstractions$
    be a set of abstractions of $\Pi$.
    The set of solutions for
    constraints~(\ref{constraint:tcp-goal})--(\ref{constraint:tcp-cost-partitioning})
    that maximize $\sum_{\alpha\in\abstractions} h(\alpha(s))$ (projected to $h$)
    is equal to the set of solutions for
    constraints~(\ref{constraint:goal-awareness})--(\ref{constraint:consistency})
    for $\features_\abstractions$
    that maximize $\varphi(s)$.
    \label{proposition:solution-sets-tcp-pot}
\end{proposition}

\begin{proof}
The important observation for this proof is that we can write $\varphi(s) =
\sum_{f\in\features_\abstractions} w(f)[s \models f]$ as $\varphi(s) =
\sum_{\alpha\in\abstractions} w(\alpha(s))$.

Assume we have an optimal solution to
constraints~(\ref{constraint:tcp-goal})--(\ref{constraint:tcp-cost-partitioning})
and set $w = h$. Obviously, constraint~(\ref{constraint:tcp-goal})
implies constraint~(\ref{constraint:goal-awareness}): if all features
that are present in the goal state have a weight of $0$, then
$\varphi(\goal) = 0$. Summing
constraint~(\ref{constraint:tcp-consistency}) over all abstractions for
a given transition, results in $\sum_{\alpha \in \abstractions}
h(\alpha(s)) - h(\alpha(s')) \le \sum_{\alpha \in \abstractions}
c^\alpha(\trans{s}{o}{s'})$. Together with
constraint~(\ref{constraint:tcp-cost-partitioning}) this implies
constraint~(\ref{constraint:consistency}).

For the other direction, assume we have an optimal solution to
constraints~(\ref{constraint:goal-awareness})--(\ref{constraint:consistency}),
\ie, that $\varphi$ is an admissible and consistent potential heuristic
with weight function $w$. Set $h = w$ and $c^\alpha(\trans{s}{o}{s'}) =
w(\alpha(s)) - w(\alpha(s'))$.

We also assume that $w(\alpha(\goal)) = 0$ for all $\alpha \in
\abstractions$. If this is not the case, consider the weight function
$w'$ with $w'(\alpha(s)) = w(\alpha(s)) - w(\alpha(\goal))$ for all
$\alpha \in \abstractions$ and $s\in\states$. Let $\varphi'$ be the
induced potential function with $\varphi'(s) = \sum_{\alpha \in
\abstractions} w'(\alpha(s)) = \varphi(s) - \varphi(\goal)$ We know
from constraint~(\ref{constraint:goal-awareness}) that $\varphi(\goal)
\le 0$, so $\varphi'$ dominates $\varphi$. It also still satisfies
constraints~(\ref{constraint:goal-awareness}) and
(\ref{constraint:consistency}), so $\varphi'$ must be an
optimal solution that we can use instead of $\varphi$. Thus, constraint~(\ref{constraint:tcp-goal}) is
satisfied.

Constraint~(\ref{constraint:tcp-consistency}) is trivially satisfied.
Constraint~(\ref{constraint:tcp-cost-partitioning}) is also satisfied,
which can be seen by replacing $c^\alpha(\trans{s}{o}{s'})$ by its
definition: the inequality $\sum_{\alpha \in \abstractions}
w(\alpha(s)) - w(\alpha(s')) = \varphi(s) - \varphi(s') \le \cost(o)$
is identical to constraint~(\ref{constraint:consistency}).
\end{proof}

\begin{theorem}
    \label{thm:relation-pot-tcp}
    Let $\Pi$ be a planning task and $\abstractions$ a set of
    abstractions for $\Pi$. For every state $s$ of $\Pi$, let
    $\hpot_{\features_\abstractions,\max s}$ be an admissible and
    consistent potential function with maximal value for $s$. Then
    $\hpot_{\features_\abstractions,\max s}(s) = \hTCP_\abstractions(s)$.
\end{theorem}


\section{Evaluation}

We implemented one- and two-dimensional admissible potential heuristics
(called $\hpot_1$ and $\hpot_2$ in the following) in the Fast Downward
planning system \cite{helmert-jair2006} and evaluated them on the tasks
from the optimal tracks of IPC 1998--2014 using limits of 2\,GB and
24~hours for memory and run time. We use this fairly high time limit,
as we are focusing on heuristic quality, not evaluation time. The
linear programs for $\hpot_2$ can become quite big, but we are
confident that approximations with smaller representations exist.

Our first concern is heuristic accuracy in the initial state. In 437
out of 696 cases (63\%) where we could determine both $\hpot_2(\init)$
and $\hopt(\init)$, the $\hpot_2$ value is perfect. These perfect
values occur in 42 of the 51 domains. The first plot in
Figure~\ref{fig:initial-h} shows initial heuristic values for
$\hpot_2(\init)$ vs.\ $\hopt$. There are still some heuristic values
that are far from the optimum, suggesting that using features of size
three or larger might be necessary in some domains.

The second plot in Figure~\ref{fig:initial-h} compares initial
heuristic values of $\hpot_1$ and $\hpot_2$. As expected, larger
features frequently achieve higher heuristic values and the initial
heuristic value is perfect less frequently. Out of the 696 tasks
mentioned above, only 110 (16\%) have a perfect $\hpot_1$ value.
There are also several cases in which $\hpot_1(\init)$ is 0, while $\hpot_2$
reports values as high as 48.
The last plot shows the number of expansions in an $\astar$ search.
Using $\hpot_2$ almost always leads to fewer expansions than $\hpot_1$. The exceptions
are mostly from the domain Blocksworld. The tasks on the x-axis of the
plot are those where $\hpot_2$ is perfect.

Theorem~\ref{thm:relation-pot-tcp} shows that $\hpot_2$ computes a TCP that is optimal in the
initial state. We know that there are tasks where the optimal TCP
results in a higher heuristic value than the optimal OCP. But does this
also occur often in practice? To answer this question, we implemented optimal
cost partitioning for projections and ran it for the set of
abstractions that contains all projections to one and two variables; the resulting heuristic is
denoted by $\hOCP_2$.
Out of the $724$ cases where we can compute both
$\hpot_2(\init)$ and $\hOCP_2(\init)$, $\hpot_2(\init)$ is higher than
$\hOCP_2(\init)$ in 312 cases (43\%). In 329 of the remaining cases
(46\%) $\hOCP_2(\init)$ is already perfect, so $\hpot_2$ cannot be
higher.
The advantage of TCP over OCP thus seems to be ubiquitous in practice.

Comparing an $\astar$ search using $\hOCP_2$ to one using $\hpot_2$ we
notice two conflicting effects. First, as we have seen, the potential
heuristic is perfect for more instances and even if it is not, its
value often exceeds the value of $\hOCP_2$. Second, the linear program of
$\hpot_2$ is evaluated only once, while the optimal OCP
is re-computed in every state. The latter is beneficial in some cases ($\hpot_2$ is guaranteed to be
no worse than $\hOCP_2$ only at the initial state), but this
re-computation takes too much time in practice.
Diversification approaches \cite{seipp-et-al-ijcai2016} can be used
to boost the value of potential heuristics on states other than the initial
state.

\section{Conclusion}

We have shown that highly accurate potential heuristics of dimension
2 can be computed in polynomial time. Optimized in each state, these
heuristics correspond to an optimal TCP over atomic and binary
projections, which dominates the optimal OCP.
Generating admissible potential heuristics of higher dimension is not
tractable in general, but possible if the
variables are not too interdependent.

To make higher-dimensional potential heuristics more practically
relevant, methods to cheaply approximate the optimal TCP are needed.
One promising direction is to look into a way of identifying a good
subset of features, because usually not all features are necessary to
maximize the heuristic value. Features could be statically selected, or
learned dynamically during the search. Using diversification techniques
to select a set of multiple admissible potential heuristics also sounds
promising.

\section{Acknowledgments}

This work was supported by the European Research Council as part of the
project ``State Space Exploration: Principles, Algorithms and
Applications''.

\bibliographystyle{aaai}
\bibliography{abbrv,literatur,crossref}

\end{document}